\documentclass{article}

\usepackage[final]{neurips_2023}

\usepackage[utf8]{inputenc} 
\usepackage[T1]{fontenc}    

\usepackage{url}            
\usepackage{booktabs}       
\usepackage{amsfonts}       
\usepackage{nicefrac}       
\usepackage{microtype}      
\usepackage{xcolor}         

\usepackage{graphicx}
\usepackage{wrapfig}
\usepackage{tikz}
\usepackage{subfigure}
\usepackage{float}
\usepackage{caption}
\usepackage{enumerate}
\usepackage{array}
\usepackage{amssymb}
\usepackage[ruled, linesnumbered, boxed]{algorithm2e}
\usepackage{hyperref}
\usepackage{todonotes}

\usepackage{hyperref}
\usepackage{multirow}
\usepackage{makecell}
\usepackage{bbding}
\usepackage{mathrsfs}
\usepackage{mathtools}

\usepackage{chngcntr}
\counterwithin{table}{section}
\counterwithin{figure}{section}

\usepackage[nameinlink]{cleveref}
\Crefname{figure}{Figure}{Figures}
\crefname{figure}{Figure}{Figures}
\crefname{example}{Example}{Example}
\crefname{theorem}{Theorem}{Theorem}
\crefname{corollary}{Corollary}{Corollary}
\crefname{lemma}{Lemma}{Lemma}
\crefname{proposition}{Proposition}{Proposition}
\crefname{assumption}{Assumption}{Assumption}
\crefname{section}{Section}{Section}
\crefname{algorithm}{Algorithm}{Algorithm}

\usepackage{amsthm,thmtools}
\declaretheorem[name=Theorem,numberwithin=section]{theorem}
\declaretheorem[name=Definition,style=definition]{definition}

\declaretheorem[name=Proposition,numberlike=theorem]{proposition}

\declaretheorem[name=Remark,style=definition,numberwithin=section]{remark}

\usepackage{todonotes}

\usepackage{amsmath,amsfonts,bm}

\def\eqref#1{equation~\ref{#1}}

\def\1{\bm{1}}

\DeclareMathAlphabet{\mathsfit}{\encodingdefault}{\sfdefault}{m}{sl}
\SetMathAlphabet{\mathsfit}{bold}{\encodingdefault}{\sfdefault}{bx}{n}

\title{CL4KGE: A Curriculum Learning Method for Knowledge Graph Embedding}

\author{%
  Yang Liu, Chuan Zhou \\
  Academy of Mathematics and Systems Science\\
  Chinese Academy of Sciences\\
  Beijing, China \\
  \texttt{  \{liuyang2020, zhouchuan\}@amss.ac.cn } \\
  \And
  Peng Zhang \\
  Cyberspace Institute of Advanced Technology\\
  Guangzhou University \\
  Guangzhou, China \\
  \texttt{p.zhang@gzhu.edu.cn} \\
  \And
  Yanan Cao \\
  Institute of Information Engineering \\
  Chinese Academy of Sciences\\
  Beijing, China \\
  \texttt{caoyanan@iie.ac.cn} \\
  \And
  Yongchao Liu \\
  Ant Group \\
  Hangzhou, China \\
  \texttt{yongchao.ly@antgroup.com} \\
  \And
  Zhao Li, Hongyang Chen \\
  Research Center for Graph Computing\\
  Zhejiang Lab \\
  Hangzhou, China \\
  \texttt{lzjoey@gmail.com, dr.h.chen@ieee.org} \\
}

\begin{document}

\maketitle

\begin{abstract}
Knowledge graph embedding (KGE) constitutes a foundational task, directed towards learning representations for entities and relations within knowledge graphs (KGs), with the objective of crafting representations comprehensive enough to approximate the logical and symbolic interconnections among entities. In this paper, we define a metric Z-counts to measure the difficulty of training each triple ($<$head entity, relation, tail entity$>$) in KGs with theoretical analysis. Based on this metric, we propose \textbf{CL4KGE}, an efficient \textbf{C}urriculum \textbf{L}earning based training strategy for \textbf{KGE}. This method includes a difficulty measurer and a training scheduler that aids in the training of KGE models. Our approach possesses the flexibility to act as a plugin within a wide range of KGE models, with the added advantage of adaptability to the majority of KGs in existence. The proposed method has been evaluated on popular KGE models, and the results demonstrate that it enhances the state-of-the-art methods. The use of Z-counts as a metric has enabled the identification of challenging triples in KGs, which helps in devising effective training strategies.
\end{abstract}



\section{Introduction}\label{sec: introduction}

Knowledge Graph Embedding (KGE) is a powerful technique for modeling and comprehending complex knowledge graphs. KGE methods aim to learn low-dimensional representations of entities (nodes) and their relationships (edges) in a knowledge graph. The goal is to encode the structural and semantic information present in the graph into the low-dimensional representations, which are then used for downstream applications such as reasoning \citep{fang2023mln4kb}, recommendation system \citep{zhang2016collaborative,wang2018dkn,ma2019jointly,wang2019kgat,lin2023towards,zhang2024meta}, question answering \citep{huang2019knowledge,lukovnikov2017neural}, and querying \citep{chen2022fuzzy,lin2019guiding}.
{A significant number of contemporary KGE models, as encapsulated by the prevailing paradigm \citep{kamigaito2022comprehensive}, frequently overlook the heterogeneity in training difficulties encountered among triplets, as well as the potential richness of information that is embedded within them. This oversight may limit the depth of learning and the effectiveness of the resulting embeddings in capturing complex relationships.} This observation has sparked our interest in delving into the curriculum learning dilemma within KGE.

Curriculum learning (CL) is a training strategy proposed by Bengio \citep{bengio2009curriculum} for learning tasks with difficult training samples, which is to rank a sequence of samples from easy learning to difficult learning. Since the easy learning samples carry more useful information, curriculum learning enables task learning more efficiently and effectively. However, from the viewpoint of KEG \citep{bordes2013translating, sun2019rotate}, the training samples are triplets that are not independent but closely related together. An effective way is to rank these training triplets from easy to difficult. The easy-training triplets refer to the triplets with recognizable feature representations. It is easier for KGE models to identify a precise decision boundary from easy triples. On the contrary, the difficult triplets refer to the ones without recognizable characteristics. This kind of triplet potentially confuses the model convergence and prevents a KGE model from successfully training a good model. Consequently, crafting a metric capable of gauging the difficulty of samples within a knowledge graph, particularly one abundant in intricate topology and semantic framework, presents a formidable challenge.

To address this challenge, this paper proposes a new \textbf{C}urriculum \textbf{L}earning method catered for \textbf{K}nowledge \textbf{G}raph \textbf{E}mbedding (\textbf{CL4KGE}). {In particular, we introduce a metric named Z-counts to quantify and evaluate the difficulty of triplets. This metric enables the ranking of training triplets, thereby guiding the training process to progress from simpler to more complex triplets. Furthermore, we propose a curriculum learning strategy that consists of two main components. The first component involves a difficulty assessment based on Z-counts, while the second component features a training scheduler that leverages the structure of the knowledge graph. Our empirical analysis includes extensive experiments on various benchmark datasets, and the findings highlight the superior performance of our proposed method, showcasing a significant improvement over existing approaches.}

The main contributions of this work are summarized as follows:


\begin{itemize}
    \item {\bf Difficulty Metric: Z-counts.}\; We propose a new metric \emph{Z-counts} to measure the training-triplet difficulty in a knowledge graph. We theoretically analyze the metric and explain how to solve the training difficulty based on the \emph{Z-counts} (see \Cref{def: Z}). 
    \item {\bf Curriculum Learning: CL4KGE.}\; We design a curriculum learning method \textbf{CL4KGE} (see \Cref{alg:clkg}) for KGE based on \emph{Z-counts}. We demonstrate that the framework can scale well with the increasing number of relations in a knowledge graph. By ranking training triplets in a difficulty-ascending manner, our method can improve   KGE methods without increasing time complexity.
    \item {\bf Experimental Evaluation.}\; We demonstrate the superiority of the proposed method through extensive experiments on tasks including link prediction and triple classification on various datasets. Experimental results show that CL4KGE used as a plugin successfully enhances the popular KGE models and achieves state-of-the-art results.
\end{itemize}


\section{Related Works}\label{sec: related works}

In this section, we summarize the previous works about knowledge graph embedding and curriculum learning. To the best of our knowledge, this is the first attempt to design a curriculum learning framework for knowledge graph embedding.

\subsection{Knowledge Graph Embedding.} Knowledge Graph Embedding (KGE) is a powerful tool for modeling and understanding complex knowledge graphs. Typical models include TransE \citep{bordes2013translating},  TorusE \citep{ebisu2018toruse},  RotatE \citep{sun2019rotate},  ConvE \citep{dettmers2018convolutional}, MQuadE \citep{yu2021mquade}, RESCALE \citep{nickel2011three}, DistMult \citep{yang2014embedding}, and ComplEx \citep{trouillon2016complex}.  KGE methods aim to learn low-dimensional representations of entities (nodes) and relations (edges) in a knowledge graph. These embeddings can then be used for downstream applications {such as link prediction} \citep{bordes2013translating, yu2021mquade, fang2023mln4kb, shengyuan2024differentiable}, knowledge graph completion \citep{bordes2013translating} and entity alignment \citep{chen2024entity}. Link prediction is a fundamental task in knowledge graph and it has a wide range of applications including healthcare \citep{almansoori2012link,almansoori2011link}, education \citep{liu2022locality}, and social analysis \citep{ferreira2016enhancing}. Several approaches have been proposed to learn knowledge graph embeddings and link predictions, including methods that rely on matrix factorization \citep{yu2021mquade}, graph neural networks \citep{2020Composition, lin2022structure, lin2024graph, lin2021disentangled, zhu2024rhgnn, liu2024combinatorial, liu2024decision}, and probabilistic modeling \citep{pujara2013knowledge}. There are also some methods that combine multiple approaches \citep{chen2020review} to improve the accuracy of the models. 

\begin{table}[!t]
    \small
    \centering
    \begin{tabular}{l l}
     \toprule
        Notation & Description  
     \\ \midrule
        $\mathcal{H}$ & knowledge graph  \\
        $\mathcal{E}$ & entity set  \\
        $\mathcal{R}$ & relation set \\
        $e,e_1,e_2,\cdots$ & entity \\
        $V,V_1,V_2,\cdots$ & node \\
        $r,r_1,r_2,\cdots$ & relation \\
        $h, h_1, h_2,\cdots$ & head entity \\
        $t, t_1, t_2,\cdots$ & tail entity \\
        $\mathbf{T}_{\text{train}}, \mathbf{T}_{\text{valid}}, \mathbf{T}_{\text{test}}$ & training, validation, test set  \\
     $\mathcal{C}$  &  training criteria \\
     $Q_1, \cdots, Q_t, \cdots, Q_T$ & training samples in each epoch $t$ \\
     $g(\cdot)$ & pacing function\\
     $p_0$ & initial percentage\\
     \bottomrule
    \end{tabular}
    \smallskip
    \caption{ Notations and their corresponding descriptions. } \label{table:notation}
    \end{table}



\subsection{Curriculum Learning} Curriculum learning \citep{bengio2009curriculum} is a powerful technique for training machine learning models, especially in complex tasks. Providing a sequence of tasks to a model allows it to learn more efficiently and effectively. Knowledge Graphs provide a structured way to represent knowledge, easily accessible to machines. This makes them especially well suited for use with curriculum learning. Curriculum learning enhances generalization ability and directs the model toward a better parameter space, according to previous works \citep{bengio2009curriculum,weinshall2020theory} in various domains including Graph \citep{wang2021curgraph,wei2022clnode,liu2023decision}, NLP \citep{xu2020curriculum,cirik2016visualizing}, CV \citep{zhang2021flexmatch,almeida2020low}, Medical \citep{burduja2021unsupervised,liu2022competence,alsharid2020curriculum}, Speech \citep{ristea2021self,wang2020curriculum,zhang2019automatic}, even Robotics \citep{florensa2017reverse,milano2021automated,manela2022curriculum}. To the best of our knowledge, no work has yet attempted to apply curriculum learning to link prediction in knowledge graph.

\begin{figure*}[t]
    \centering
    \includegraphics[width=0.8\linewidth]{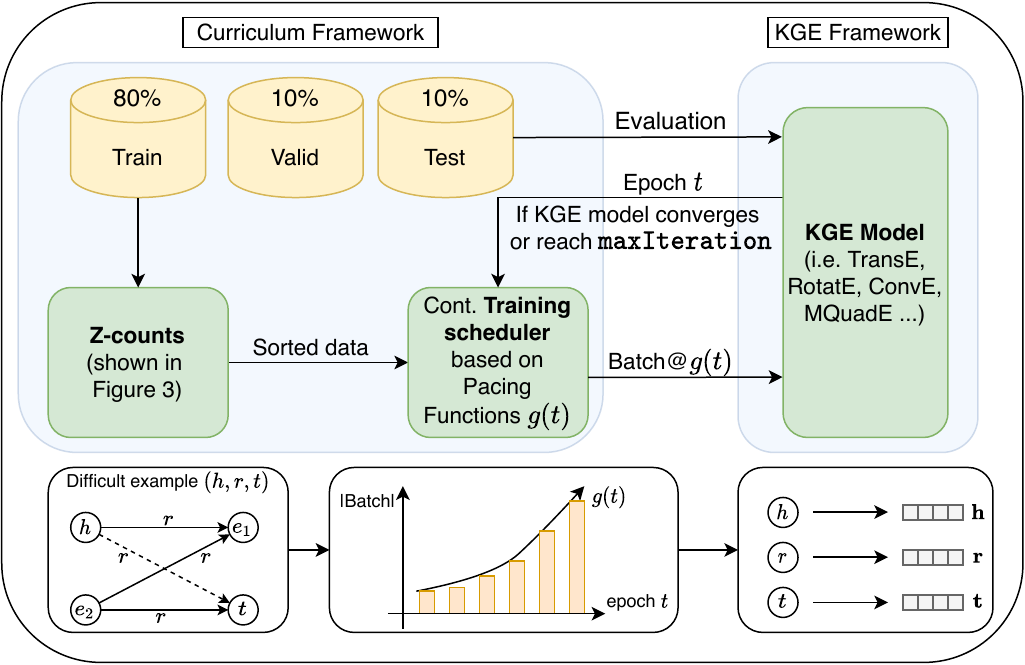}
    \caption{An overview of CL4KGE framework which contains two parts: a curriculum learning framework and KGE methods. The three small diagrams below are an illustration of the corresponding three modules. The Z-counts based curriculum framework leads to a more effective training strategy. At the bottom left is the Z-shaped phenomenon which means if $h \rightarrow e_1, e_2 \rightarrow e_1, e_2 \rightarrow t$ holds, $h \rightarrow t$ is likely to be true. We design a metric \emph{Z-counts} according to the Z-shaped phenomenon.}
    \label{fig: Cl-KGX}
\end{figure*}

\section{Problem Setup and Background}\label{sec: background}


In this section, we state the problem formulation and provide a succinct theoretical basis for the problem we study, including some preliminaries of knowledge graph and mathematics statement of curriculum learning. Additionally, we go through the notations involved in this work in \Cref{table:notation}. We denote scalars, vectors and matrices with lowercase letters, bold lowercase letters, and bold uppercase letters, respectively.

\subsection{Knowledge Graph Embedding}
{A knowledge graph can be described as a set of triplets (head entity, relation, tail entity) denoted as $\{(h,r,t) \mid h, t \in \mathcal{E}, r \in \mathcal{R}\}$.} As for training strategy, KGE methods would introduce a score function \citep{socher2013reasoning} defined in \Cref{def:score} for each triple, and the loss function is composed for the model update.

{
\begin{definition}[Score Function]\label{def:score}
    For each triplet $(h, r, t)$ in a knowledge graph $\mathcal{H}$, the score function $\mathbf{f}: \mathbb{R}^d \times \mathbb{R}^d \times \mathbb{R}^d \to \mathbb{R}^+ \cup \{0\}$ is a non-negative real-valued function. We define it as:
    \begin{equation}
    \mathbf{f}(\mathbf{h}, \mathbf{r}, \mathbf{t}) = \mathsf{r}(\mathbf{h}, \mathbf{t}) 
    \longrightarrow
     \left\{
\begin{aligned} 
& 0, \quad \text{if } (h, r, t) \text{ holds}, \\
&+\infty, \quad \text{otherwise},
\end{aligned}
\right.
    \end{equation}
    where $\mathbf{h}$, $\mathbf{r}$, and $\mathbf{t}$ are the embeddings of entities $h$, $r$, and $t$, respectively, and $\mathsf{r}$ is a relation function mapping from $\mathbb{R}^d \times \mathbb{R}^d$ to $\mathbb{R}^+ \cup \{0\}$.
\end{definition}
}

Given two elements in a triplet $(h,r,t)$ with a score function $f(h,r,t)$, the task of link prediction is to add the most logical element with the highest confidence level in the remaining position. Translation distance-based approaches often describe relations using geometric transformations of entities, and they assess the plausibility of fact triples by comparing the distances between entity embedding following relation transformations. Also, each translation distance-based KG method will define a unique score function. However, deep learning-based KGE methods follow a different way to encoder the triplet which could use some popular neural networks including CNNs \citep{dettmers2018convolutional,2015RelationNguyen}, RNNs \citep{2015RelationZhang}, and GNNs \citep{2020Composition}.

\begin{figure}[t]
    \centering
    \includegraphics[width=0.8\linewidth]{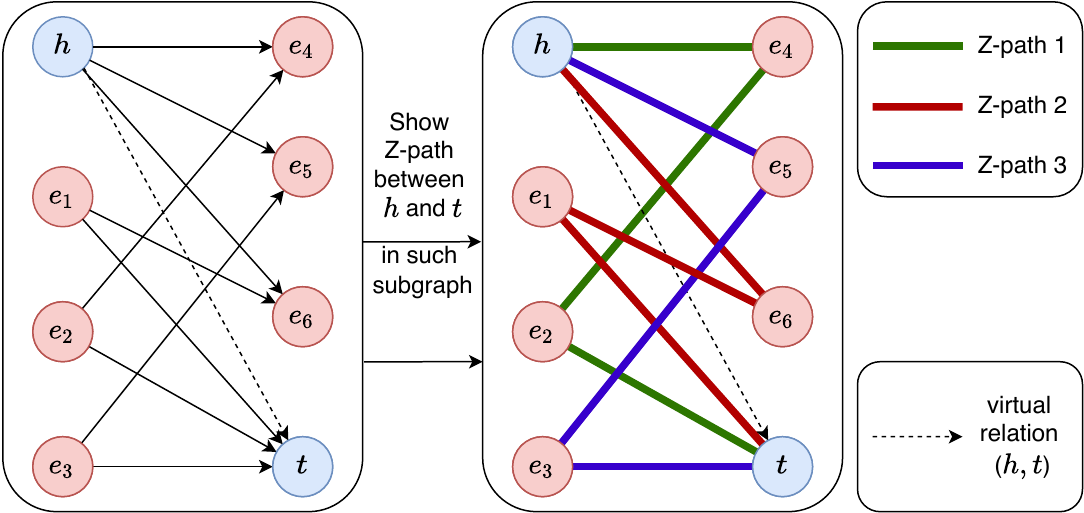}
    \caption{This illustration shows an example of Z-counts between $h$ and $t$. In this graph, there are three Z-path between $h$ and $t$ and the responding Z-counts is 3. We label different Z-paths with different colors in the right hand-side graph.}
    \label{fig: z-counts}
\end{figure}

{
\begin{remark}
    The task of link prediction involves predicting the missing elements in such triplets to infer new facts and relationships in the knowledge graph. Specifically, if either $h$ or $t$ is missing, the goal is to identify the most plausible entity that completes the triplet based on the embeddings and learned patterns in the graph. The primary motivation for link prediction is that knowledge graphs are typically incomplete; not all possible relationships between entities are known or present. Therefore, link prediction aims to discover these missing links, enhancing the knowledge graph's completeness and utility.
    For instance:
    \begin{itemize}
        \item Predicting the tail entity $t$ given $(h, r, ?)$
        \item Predicting the head entity $h$ given $(?, r, t)$
    \end{itemize}
    The plausibility of a predicted link is usually quantified using a scoring function that assesses how likely a given triplet $(h, r, t)$ is to be true within the learned embedding space of the knowledge graph.
\end{remark}
}

\subsection{Curriculum Learning}
Curriculum learning (CL) is a training strategy that mimics the suitable learning sequence seen in human curricula by training a machine learning model from easy to difficult data. By utilizing a curriculum to train the model, curriculum learning reduces the negative effects of challenging samples. A curriculum sequence contains a series of training criteria $<Q_1, \cdots, Q_t, \cdots, Q_T>$ over $T$ epochs. Each training criteria $Q_t$ is a subset of the training set. With the increase of training epoch, the number of hard samples in $Q_t$ is gradually increased. We design the curriculum learning framework in two parts, difficulty measure and training scheduler, which will be introduced in details in \Cref{sec: method} and \Cref{sec: training scheduler} respectively. In our framework, the Z-counts ranks the difficulty of each triplet in the training epochs, and the training scheduler generates $Q_i$ for each batch to train the model. The overview of our proposed framework can be found in \Cref{fig: Cl-KGX}.

Before we introduce the details of our proposed curriculum learning for knowledge graph embedding, we decide to introduce the original CL framework. For a more rigorous presentation and easy understanding, we give the original definition of CL proposed by Bengio \citep{bengio2009curriculum}.

\begin{definition}[Curriculum Learning \citep{bengio2009curriculum}]\label{def: cl}
    A curriculum is a sequence of training criteria over T training steps: $\mathcal{C} = <Q_1, \cdots, Q_t, \cdots, Q_T>$ and each criterion $Q_t$ is a re-weighting of the target training distribution $P(z)$:
    \begin{equation*}
        Q_t \propto W_t(z)P(z) \quad \forall \text{ sample } z \in \mathbf{T}_{\text{training}}
    \end{equation*}
    Also, these should satisfy the following three conditions:
    \begin{itemize}
        \item The entropy of distributions gradually increases, $H(Q_t)<H(Q_{t+1})$,
        \item The weight for any example increases, $W_t(z)\leq W_{t+1}(z) \quad \forall \text{ sample } z \in \mathbf{T}_{\text{training}}$,
        \item $Q_T(z) = P(z)$.
    \end{itemize}
\end{definition}

\section{Method: CL4KGE}\label{section:CL4KGE}

This section provides a comprehensive exposition of our proposed framework CL4KGE, which is intricately grounded in curriculum learning principles. It comprises two integral components: a difficulty measurer (refer to \Cref{sec: method}) and a training scheduler (refer to \Cref{sec: training scheduler}).

\subsection{Difficulty measurer: Z-counts}\label{sec: method}


In this section, we introduce the metric \textbf{Z-counts} used to measure the quality of the triplets, whose intuitive idea is shown in \Cref{fig: z-counts}, before we present the details of the proposed CL4KGE framework. Our strategy's main idea is to enhance the performance of the backbone KGE methods by gradually incorporating the training samples into the learning process as it progresses from easy to difficult. There are always some Z-shaped pathways for each pair of $(h,t)$ to all of the relation $r$, regardless of whether it is really present, in every knowledge graph that consists of triple facts $\{h_i,r_i,t_i\}_{i=1}^N$. To be more rigorous, we give the definition of \textbf{Z-counts} below:

\begin{figure*}[t]
    \centering
    \mbox{
        \includegraphics[width=0.3\linewidth,trim={0 50pt 0 10pt}]{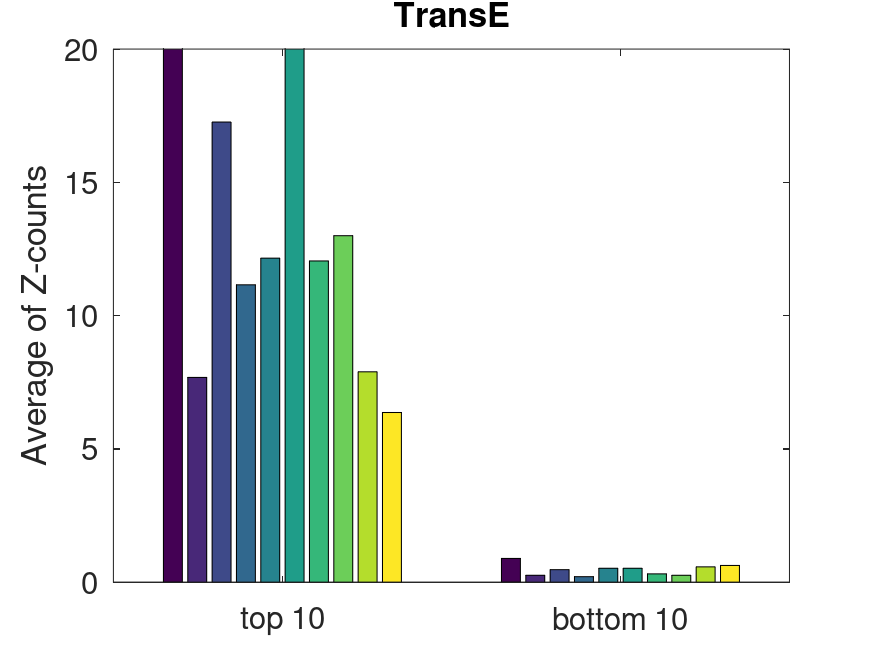}
        \includegraphics[width=0.3\linewidth, trim={0 50pt 0 10pt}]{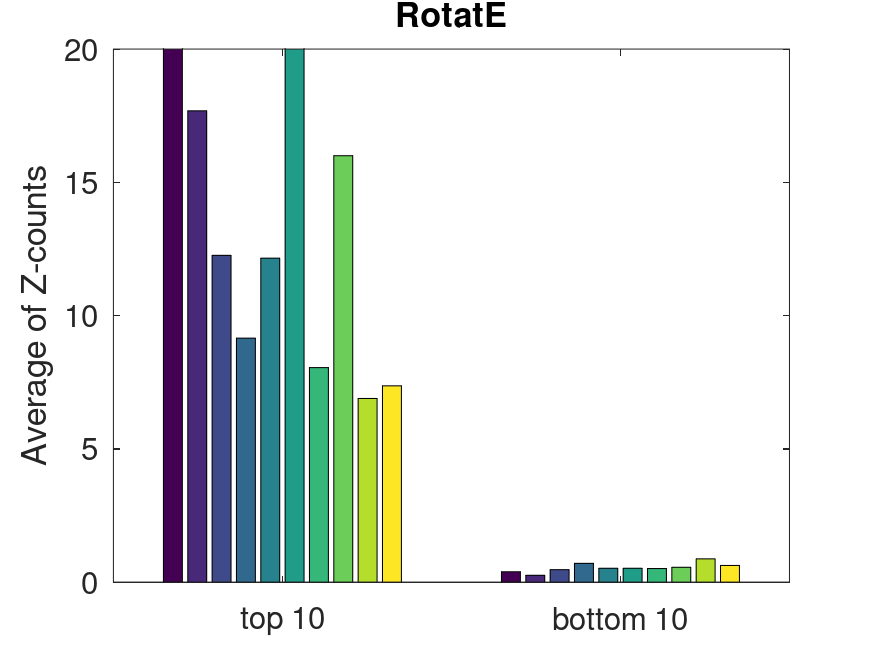}
        \includegraphics[width=0.3\linewidth, trim={0 50pt 0 10pt}]{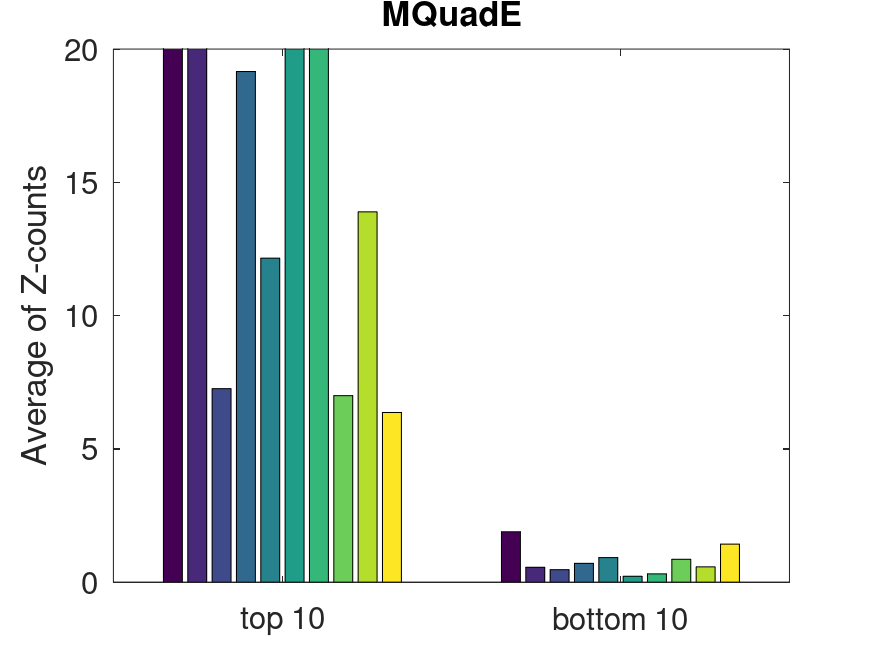}
    }
        \caption{The visualization of the Z-counts of TrasnE, RotatE, and MQuadE for the FB15k datasets. Top-10 means the rank is smaller than 10 during evaluation and bottom-10 means the opposite. The y-axis is the average of Z-counts which only shows the interval (0, 20).} \label{fig: visualization} 
    \end{figure*}


{
\begin{definition}[Z-path, Z-counts]\label{def: Z}
    For each triplet $(h, r, t)$ in a knowledge graph $\mathcal{H}$, we define a Z-path and the corresponding Z-counts as follows:
    A Z-path exists if there are intermediate entities $e_1$ and $e_2$ such that the following triplets are present in the knowledge graph: $(h, r, e_1)$, $(e_1, r, e_2)$, and $(e_2, r, t)$. The Z-counts for a triplet $(h, r, t)$ is the total number of such Z-paths.
\end{definition}
}

{
\begin{definition}[Separable function]\label{def: separable}
    A function $F(x_1, x_2 \ldots, x_n)$, defined on a region $D \subset \mathbb{R}^n$, is said to be separable with respect to the variable $x_1$ if it can be written as a summation of two functions, defined on $D$, with one function $G$ depending only on $x_1$, whereas the other function $H$ is independent of $x_1$, i.e.
    \begin{align}
        F(x_1, x_2 \ldots, x_n) = G(x_1) + H(x_2 \ldots, x_n)
    \end{align}
\end{definition}
}

\begin{proposition}\label{proposition}
Given a KGE method with a score function $\mathsf{r}(\mathbf{h},\mathbf{t})$ which is separable respect to $\mathbf{h}$ and $\mathbf{t}$, we have $\mathsf{r}(\mathbf{h},\mathbf{t}) = 0$ if there exists Z-path between $\mathbf{h}$ and $\mathbf{t}$.
\end{proposition}

\begin{proof}
    If $\mathsf{r}(\mathbf{h},\mathbf{t})$ is separable, we have $\mathsf{r}(\mathbf{h},\mathbf{t}) = \mathsf{r_1}(\mathbf{h}) + \mathsf{r_2}(\mathbf{t})$ for some $\mathsf{r_1}(\cdot)$ and $\mathsf{r_2}(\cdot)$. For the four entities $h,e_2,e_3,t$, assuming that $(h,r,e_2)$, $(e_3,r,e_2)$, $(e_3,r,t)$ hold, now we prove that it is very likely $(h,r,t)$ also holds. 
    
    Let $\mathbf{h},\mathbf{e_2},\mathbf{e_3},\mathbf{t}$ be the embeddings of $h,e_2,e_3,t$.
    Since $(h,r,e_2)$ holds, we have
    \begin{equation*}
    \mathsf{r}(\mathbf{h},\mathbf{e_2})=\mathsf{r_1}
    (\mathbf{h})+\mathsf{r_2}(\mathbf{e_2}) \approx 0,
    \end{equation*}
    and similarly,
    \begin{equation*}
    \mathsf{r_1}(\mathbf{e_3}) + \mathsf{r_2}(\mathbf{e_2}) \approx 0,\qquad
    \mathsf{r_1}(\mathbf{e_3}) + \mathsf{r_2}(\mathbf{t}) \approx 0.
    \end{equation*}
    Then it follows that 
    \begin{align*}
    \mathsf{r}(\mathbf{h},\mathbf{t}) &= 
    \mathsf{r_1}(\mathbf{h}) + \mathsf{r_2}(\mathbf{t}) \\
     &= [\mathsf{r_1}(\mathbf{h}) + \mathsf{r_2}(\mathbf{e_2})] - [\mathsf{r_1}(\mathbf{e_3})
     + \mathsf{r_2}(\mathbf{e_2})] +[\mathsf{r_1}(\mathbf{e_3}) + \mathsf{r_2}(\mathbf{t})]
      \\ &\approx 0
    \end{align*}
        i.e., $(h,r,t)$ holds which completes the whole proof.
\end{proof}

{
\begin{remark}
    It is important to note that the separability condition (refer to \Cref{def: separable}) for score functions is relatively easy to satisfy and not overly restrictive. Specifically, we can rewrite the score functions of TransE and RotatE to conform to the separable condition. For instance, $s_{\text{TransE}}(h,r,t) = h + r - t = (h + r) - t = f_1(h,r) - g(t,r)$ and $s_{\text{RotatE}}(h,r,t) = h \circ r - t = (h \circ r) - t = f_2(h,r) - g(t,r)$, where $f_1(h,r) = h + r$, $f_2(h,r) = h \circ r$, and $g(t,r) = t$. We provide a list of typical methods with separable score functions in \Cref{table: KGEmethods}. Additionally, Knowledge Graph Embedding (KGE) models with bilinear score functions can satisfy the separable condition if $2\mathbf{h}^T \mathbf{R} \mathbf{t} = -|\mathbf{R} \mathbf{t} - \mathbf{h}|^2 + |\mathbf{h}|^2 + |\mathbf{R} \mathbf{t}|^2$. Consequently, several KGE models, including DisMult, ComplEx, DihEdral, QuatE, SEEK, and Tucker, also meet this condition.
\end{remark}
}

{
\begin{remark}
    For any entities $h$ and $t$, the larger the Z-counts between $h$ and $t$ is, the more likely that the entities $h$ and $t$ are connected by this KGE method. The proof shows that if there exists at least one Z-path (i.e., Z-counts $>0$), then the score function $\mathsf{r}(\mathbf{h},\mathbf{t})$ is likely close to zero, indicating a strong likelihood that $(\mathbf{h},\mathbf{r},\mathbf{t})$ is a valid triplet. This aligns with the statement because a higher Z-count means there are more Z-paths, which collectively increase the likelihood that $\mathsf{r}(\mathbf{h},\mathbf{t}) \approx 0$.
\end{remark}
}

Now we use the conclusion of \Cref{proposition} and experimental numerical analysis to explain why we can use the Z-counts to measure the difficulty of a triplet. By \Cref{proposition}, given a KGE method with a separable score function, the more Z-path between entity $h$ and entity $t$ is, the more likely they are connected by this KGE method. Under this observation, we doubt whether the Z-counts plays an critical role in the inference of this KGE method. In other words, we want to know whether this KGE model would predict a triplet as a true fact as long as its Z-counts is large, regardless of whether it is reasonable. If this is the case, the triples with large Z-counts will not necessarily to provide recognizable characteristics. On the contrary, the triples with small Z-counts will be more likely to provide more useful features to identify a precise decision boundary.

\begin{table}[t]
    \small
    \centering
        \begin{tabular}{lcc} 
        \toprule
        Model & Score function $f(h,r,t)$ & \# Representation of param    
        \\
        \midrule
        DisMult & $-\left\langle \mathbf{h},\mathbf{r},\mathbf{t}  \right\rangle$ 
        & $\mathbf{h},\mathbf{r},\mathbf{t} \in \mathbb{R}^k$
        \\
        
        ConvE  & $g( \operatorname{vec}(g( [\mathbf{h},\mathbf{r}]*\mathbf{w} ))W )\mathbf{t}$  
        & $\mathbf{h},\mathbf{r},\mathbf{t} \in \mathbb{R}^k, \mathbf{w}\in \mathbb{R}^{m_1}, W \in \mathbb{R}^{m_2} $
        \\
        TransE & $\Vert \mathbf{h}+\mathbf{r}-\mathbf{t} \Vert$ 
        & $\mathbf{h},\mathbf{r},\mathbf{t} \in \mathbb{R}^k$
        \\
        RotatE & $\Vert \mathbf{h} \circ \mathbf{r}-\mathbf{t} \Vert$ 
        & $\mathbf{h},\mathbf{r},\mathbf{t} \in \mathbb{C}^k, \Vert \mathbf{r}_i \Vert=1$
        \\
        MQuadE & $\Vert  \mathbf{H}\mathbf{R}-\widehat{\mathbf{R}}\mathbf{T}      \Vert  $ 
        & \thead{$\mathbf{H},\mathbf{T},\mathbf{R},\mathbf{\widehat{R}} \in \mathbb{R}^{p \times p},$\\ $\mathbf{H}, \mathbf{T} \text{ are symmetric }$}
        \\  
        \bottomrule
         \end{tabular}
         \smallskip
         \caption{ Some classical models with separable score function and their parameters.} \label{table: KGEmethods}
    \end{table}

Next, we employ the experimental results to prsent a positive answer to the above doubt. 
For each dataset including B15k\citep{bollacker2008freebase}, FB15k-237 \citep{toutanova2015observed}, WN18 \citep{bordes2013translating}, and WN18RR \citep{dettmers2018convolutional}, we counted the number of Z-counts in the training set. \Cref{table:Statistics} shows the exact numbers. We can observe that the samples in FB15k and FB15k-237 datasets have more Z-counts than those in WN18 and WN18RR.

\begin{table}[t]
\small
\begin{center}
\begin{tabular}{llll}
\toprule
Dataset & \#Training & \#Validation & \#Test \\
\midrule
FB15k & 483,142 & 50,000 & 59,071 \\
FB15k-237 & 272,115 & 17,535 & 20,466 \\
WN18 & 141,442 &5,000 &5,000 \\
WN18RR & 86,835 &3,034 &3,134 \\
Yago3-10 & 1,079,040 & 5,000 & 5,000 \\
\bottomrule
\end{tabular}
\smallskip
\caption{Statistics information of benchmark datasets.}
\label{table:datasets}
\end{center}
\end{table}

\begin{table}[t]
\begin{center}
\begin{tabular}{lcccc}
\toprule
Dataset & \# Sample (Z-counts$>0$) &  Avg & Max & Percentage \\
\midrule
FB15k & 292,908 & 44 & 1381 & 0.606 \\
FB15k-237 & 167,461 & 49 & 1246 & 0.615  \\
WN18 & 16,414 & - & 107 & 0.083   \\
WN18RR & 10,033 & - & 107 & 0.116  \\
Yago3-10 & 831,939 & 42 & 1268 & 0.771\\
\bottomrule
\end{tabular}
\smallskip
\caption{Statistics information of Z-counts in each dataset, including the number of samples with non-zero Z-counts, the average of Z-counts, the max of Z-counts, and the percentage of samples with non-zero Z-counts.}
\label{table:Statistics}
\end{center}
\end{table}

We use the original TransE \citep{bordes2013translating}, RotatE \citep{sun2019rotate} and MQuadE \citep{yu2021mquade} to evaluate the testing set of the FB15k dataset. We calculate the average number of Z-counts of the top-10 and bottom-10 samples. In \Cref{fig: visualization}, we can see that the Z-counts in the top 10 are remarkably more than those in the bottom 10. It is obvious that the uneven distributions are illogical, which implies that the Z-paths should provide important information in the inference of this KGE method.

Under the above analysis, we will view a triplet with small Z-counts as an easy one and view a triplet with a big Z-count as a difficult one.



\begin{table}[t]
\begin{center}
\begin{small}
\begin{tabular}{lrr}
\toprule
Dataset & \#Entity & \# Relation \\
\midrule
FB15k & 14,951 & 1,345 \\
FB15k-237 & 14,541 & 237 \\
WN18 &  40,943 & 18 \\
WN18RR &  40,943  & 11 \\
Yago3-10 & 123,182 & 37\\
\bottomrule
\end{tabular}
\smallskip
\caption{Statistics information of benchmark datasets.}
\label{table: datasets}
\end{small}
\end{center}
\end{table}

\begin{algorithm}[!ht]  
    \caption{CL4KGE framework}\label{alg:clkg}
        \KwIn{A kowledge base $\mathcal{H} = (\mathcal{E}_{\mathcal{H}}, \mathcal{R}_{\mathcal{H}})$, the max iteration number $\texttt{maxIter}$, the pacing function $g(t)$.}
        \KwOut{The embeddings of entities and relations $(\mathcal{E}_{\mathcal{H}}, \mathcal{R}_{\mathcal{H}})$.}
        Initialize the embeddings for $(\mathcal{E}_{\mathcal{H}}, \mathcal{R}_{\mathcal{H}})$\;
        \For{$u \in \mathbf{T}_{\text{train}}$}
            {Compute the $\texttt{Difficulty}(u)$ using \Cref{def: Z}}
        Sort 
        training set according to \texttt{Difficulty} in ascending order\;
        Let $t=1$\;
        \While{not converge}{
            \If{$t < \texttt{maxIter}$}
            {
                $\lambda \gets g(t)$\;
                Batch $\gets \mathbf{T}_{\text{train}}[1:\lambda \times l]$\;
                UPDATE the embeddings of $(\mathcal{E}_{\mathcal{H}}, \mathcal{R}_{\mathcal{H}})$ via X\;
                $t \gets t+1$\;
                }
        }
\end{algorithm}

\subsection{Training scheduler}\label{sec: training scheduler}

After the difficulty measure module which measures the difficulty of each sample in training set $\mathbf{T}_{\text{train}}$, we demonstrate a curriculum learning framework to train a better KGE method. 

\subsubsection{Pacing Functions}

\begin{table}[t]
\small
\begin{center}
\begin{tabular}{lc}
\toprule
Dataset  & \thead{Initial percentage \\ $p_0$ = 1 - Percentage (in \Cref{table:Statistics})}
\\
\midrule
FB15k & 0.396 \\
FB15k-237 & 0.385  \\
WN18  &  0.917 \\
WN18RR & 0.884  \\
Yago3-10 & 0.229\\
\bottomrule
\end{tabular}
\smallskip
\caption{The initial percentage for each dataset.}
\label{table: initial}
\end{center}
\smallskip
\end{table}

\begin{table}[t]
\small
\begin{center}
\begin{tabular}{l l}
\toprule
Name & Equations \\ 
\midrule
Linear &  $p_0 + \frac{1}{p_0} \times \min (1, \lambda_{0}+\frac{1-\lambda_{0}}{T_{\text {grow }}} \cdot t)$ \\ 
Root &  $p_0 + \frac{1}{p_0} \times \min (1, \sqrt{\frac{1-\lambda_{0}^{2}}{T_{\text {grow }}} \cdot t+\lambda_{0}^{2}})$ \\ 
Root-p &   $p_0 + \frac{1}{p_0} \times \min (1, \sqrt{\frac{1-\lambda_{0}^{p}}{T_{\text {grow }}} \cdot t+\lambda_{0}^{p}})$ \\
Geometric &$ p_0 + \frac{1}{p_0} \times \min (1,2(\frac{\log _{2} 1-\log _{2} \lambda_{0}}{T_{\text {grow }}} \cdot t +\log _{2} \lambda_{0}))$ \\ 
\bottomrule
\end{tabular}
\smallskip
\caption{The equation of our pacing functions.}
\label{table: pacing function}
\end{center}
\end{table}

We propose a framework to generate the easy-to-difficult curriculum based on a continuous training scheduler. We follow the previous work \citep{wei2022clnode} to design these functions which can be viewed as intensity functions $\lambda(\cdot)$ to map the training epoch to a scale between 0 and 1, i.e. $\lambda(t): N \longrightarrow (0,1]$. According to the \Cref{def: cl}, we choose the functions as our pacing function summarized in \Cref{table: pacing function}.



\subsubsection{Pseudo Code and Complexity Analysis}

In this subsection, we first show the pseudo-code shown in \Cref{alg:clkg}. The process of CL4KGE is detailed in \Cref{alg:clkg}. Lines 1-2 describe the input and output of the algorithm. Lines 4–6 describe the Z-counts based difficulty measurer and lines 7–15 describe the process of training the backbone KGE methods with a curriculum framework. As the pseudo-code shows, CL4KGE is easy to implement and can be directly plugged into any backbone KGE method, as it only changes the training set in each training epoch (lines 11–14).




For the complexity analysis of \Cref{alg:clkg}, we mainly focus on the pre-processing of it --- the calculations of Z-counts of training samples. For a given knowledge graph $\mathcal{H}$, the space complexity is $\mathcal{O}(|\mathbf{T}_{\text{train}}|)$ and the time complexity slightly less than $\mathcal{O}(|\mathbf{T}_{\text{train}}| \times \mathcal{E}_{\mathcal{H}}^2)$. 
Since this process can obviously be handled in parallel, the time complexity could be $\mathcal{O}(\mathcal{E}_{\mathcal{H}}^2)$. Even for the large-scale benchmark dataset in \Cref{table: datasets}, we only spend just a few hours on a laptop computing Z-counts and sorting them. As for the backbone (lines 11–14 in \Cref{alg:clkg}), the complexity is the same as the baselines. Our approach don't increase the complexity of the algorithm and is even somehow improving the training process.


\section{Experiments}\label{sec: experiments}

In this section, we conduct our experiments to demonstrate the efficiency of CL4KGE. First, we introduce the basic benchmark datasets we use in our work in \Cref{exp:dataset}. Then we present a brief introduction of our compared baselines and the experimental setup in \Cref{exp:baseline} and \Cref{exp:setup}. Lastly, we will give an overview of our experimental results and discussions in the next section.

\subsection{Datasets}\label{exp:dataset}

We select the benchmark datasets used in the knowledge graph domain. The following are some descriptions of datasets with statistical information summarized in \Cref{table: datasets}. These datasets are divided into train set, validation set, and test set according to the 8:1:1 ratio.


\paragraph{\textbf{FB15k-237}} FB15k-237 \citep{toutanova2015observed} is a subset of the Freebase \citep{bollacker2008freebase} knowledge graph which contains 237 relations. The FB15k \citep{bordes2013translating} dataset, a subset of Free-base, was used to build the dataset by \citep{toutanova2015observed} to study the combined embedding of text and knowledge networks. FB15k-237 is more challenging than the FB15k dataset because FB15k-237 strips out the inverse relations.


\paragraph{\textbf{WN18}}  WN18 is a subset of the WordNet \citep{miller1995wordnet}, a lexical database for the English language that groups synonymous words into synsets. WN18 contains relations between words such as \emph{hypernym} and \emph{similar\_to}.

\paragraph{\textbf{WN18RR}} WN18RR \citep{dettmers2018convolutional} is a subset of WN18 that removes symmetry/asymmetry and inverse relations to resolve the test set leakage problem. WN18RR is suitable for the examination of relation composition modeling ability.


\paragraph{\textbf{Countries}} Countries dataset \citep{bouchard2015approximate} consists of 244 countries, 22 subregions (e.g., Southern Africa, Western Europe), and 5 regions (e.g., Africa, Americas). Each country is located in exactly one region and subregion, each subregion is located in exactly one region, and each country can have a number of other countries as neighbors.


\begin{table}[t]
\small
    \centering
    \begin{tabular}{|c|c|c|c|}
         \hline 
         \multirow{2}{*}{{Model}}& \multicolumn{3}{c|}{\text { Countries (AUC-PR) }} \\
\cline { 2 - 4 } & \text { S1 } & \text { S2 } & \text { S3 }  \\
\hline \text { DisMult } & 1.00 & 0.72 & 0.52 \\
\text { ComplEx } & 0.97 & 0.57 & 0.43 \\
\text { ConvE } & 1.00 & 0.99 & 0.86 \\
\text { RotatE } & 1.00  & 1.00  & 0.95  \\
\hline
\text { DisMult\_CL4KGE } & 1.00 & \textbf{0.82} & \textbf{0.71} \\
\text { ComplEx\_CL4KGE } & \textbf{0.99} & \textbf{0.62} & \textbf{0.45} \\
\text { ConvE\_CL4KGE } & 1.00 & \textbf{1.00} & \textbf{0.95} \\
\text { RotatE\_CL4KGE } & 1.00 & 1.00 & \textbf{0.99} \\
\hline
    \end{tabular}
    \smallskip
    \caption{Results on the Countries datasets. Other results are from ConvE and RotatE.}
    \label{tab: Countries}
\end{table}

\subsection{Baselines}\label{exp:baseline}

{Within the domain of knowledge graphs, a diverse array of methodologies has been developed, including DistMult \citep{yang2014embedding}, TransE \citep{bordes2013translating}, DihEdral \citep{yang2014embedding}, QuatE \citep{zhang2019quaternion}, TuckER, RotatE \citep{sun2019rotate}, ComplEx \citep{trouillon2016complex}, MQuadE \citep{yu2021mquade}, DensE \citep{lu2020dense}, and HousE \citep{li2022house}. We employ these well-established methods as Knowledge Graph Embedding (KGE) backbones to evaluate whether our proposed framework enhances their performance.}

\subsection{Experimental Setup}\label{exp:setup}

\paragraph{Loss function} For the convenience of training, we use the same loss function and optimization method to learn the parameters. To learn the model parameters, we apply the self-adversarial loss function suggested in MQuadE \citep{yu2021mquade},
\begin{equation}\label{equ:loss}
    \mathcal{L} = -\log \sigma(\gamma-f(h,r,t))-\sum_{i=1}^K
    p_i \log\sigma(f(h_i^{\prime},r,t_i^{\prime})-\gamma)
\end{equation}
where $\gamma > 0$ is a pre-defined margin, $\sigma$ is the sigmoid activation function, $(h_i^{\prime},r,t_i^{\prime})$ is the $i$-th negative triple and $p_i$ is the weight of the responding sample $(h_i^{\prime},r,t_i^{\prime})$ as the following equation:
\begin{equation*}
    p_i = \frac{\exp \alpha(\gamma-f(h_i^{\prime},r,t_i^{\prime}))}{\sum^K_{k=1}\exp\alpha(\gamma-f(h_k^{\prime},r,t_k^{\prime}))}
\end{equation*}
where $\alpha \in [0,1]$ is the self-adversarial temperature.

For the baselines, we adopt the stochastic mini-batch optimization methods to minimize the above loss function (see \Cref{equ:loss}) and learn the models' parameters. Also, we use the adam optimizer for parameter learning.



\paragraph{Evaluation metrics} We use the mean reciprocal rank (MRR), mean rank (MR), Hit@1, Hit@3, and Hit@10 metrics for evaluation and convenience comparison. Also, we use the filtered setting following TransE \citep{bordes2013translating} which means the triplets that appear in the training and validation set are removed from the testing set.

\begin{table}[htbp]
\tiny
    \centering
         \begin{tabular}{lcccccccccc} 
        \toprule
        & \multicolumn{5}{c}{\textbf{Original setting}} & \multicolumn{5}{c}{\textbf{CL4KGE (ours)}}\\
        \midrule
        \multirow{2}{*}{\textbf{Models}}
        &\multirow{2}{*}{\textbf{MRR}($\uparrow$)} & \multirow{2}{*}{\textbf{MR}($\downarrow$)} & \multicolumn{3}{c}{\textbf{Hits@N}($\uparrow$)}
        &\multirow{2}{*}{\textbf{MRR}($\uparrow$)} & \multirow{2}{*}{\textbf{MR}($\downarrow$)} & \multicolumn{3}{c}{\textbf{Hits@N}($\uparrow$)}
        \\
         && & \textbf{1} & \textbf{3} & \textbf{10}
         && & \textbf{1} & \textbf{3} & \textbf{10}
         \\
        \midrule
        DisMult &0.241&254&0.155&0.263&0.419
        &0.246&204&0.159&0.268&0.429 (\textcolor{red}{$\uparrow$1.0\%})\\
        ComplEX &0.247&339&0.158&0.275&0.428
        &0.255&329&0.165&0.292&0.448 (\textcolor{red}{$\uparrow$2.0\%})\\
        DihEdral &0.320&-&0.230&0.353&0.502
        &0.335&-&0.241&0.373&0.534 (\textcolor{red}{$\uparrow$3.2\%})\\
        QuatE &0.311&176&0.221&0.342&0.495
        &0.322&172&0.224&0.355&0.500 (\textcolor{red}{$\uparrow$0.5\%})\\
        TuckER &0.353&162&0.260&0.387&0.536
        &0.357&152&0.280&0.399&0.556 (\textcolor{red}{$\uparrow$2.0\%})\\
        ConvE &0.325&224&0.237&0.356&0.501
        &0.330&202&0.256&0.366&0.509 (\textcolor{red}{$\uparrow$0.8\%})\\
        TransE &0.294 &357 & - & - &0.465
        &0.306 &247 & - & - &0.495 (\textcolor{red}{$\uparrow$3.0\%   })\\
        RotatE &0.336&177&0.241&0.373&0.530
        & 0.339 & 152 & 0.281 & 0.398 & 0.542 (\textcolor{red}{$\uparrow$1.2\%   })\\
        MQuadE &0.356 &174 &0.260 &0.392 &0.549 & {0.402} & 188 & {0.293} & {0.433} & {0.572} (\textcolor{red}{$\uparrow$2.3\%   }) \\
        DensE &0.351 &161 &0.256 &0.386 &0.544
        & 0.379 &181 &0.272 &0.400 &0.555 (\textcolor{red}{$\uparrow$1.1\%})\\
        HousE  &0.361 &153 &0.266 &0.399 &0.551
        &0.366 &149 &0.286 &0.411 &0.558 (\textcolor{red}{$\uparrow$0.7\%})\\
        \bottomrule
         \end{tabular}
         \smallskip
        \caption{ Results of link prediction on the FB15k-237 dataset.} \label{table:FB15k-237}
    \end{table}

\begin{table}[htbp]
\tiny
    \centering
         \begin{tabular}{lcccccccccc} 
        \toprule
        & \multicolumn{5}{c}{\textbf{Original setting}} & \multicolumn{5}{c}{\textbf{CL4KGE (ours)}}\\
        \midrule
        \multirow{2}{*}{\textbf{Models}}
        &\multirow{2}{*}{\textbf{MRR}($\uparrow$)} & \multirow{2}{*}{\textbf{MR}($\downarrow$)} & \multicolumn{3}{c}{\textbf{Hits@N}($\uparrow$)}
        &\multirow{2}{*}{\textbf{MRR}($\uparrow$)} & \multirow{2}{*}{\textbf{MR}($\downarrow$)} & \multicolumn{3}{c}{\textbf{Hits@N}($\uparrow$)}
        \\
         && & \textbf{1} & \textbf{3} & \textbf{10}
         && & \textbf{1} & \textbf{3} & \textbf{10}
         \\
        \midrule
        ComplEX &0.941&-&0.936&0.945&0.947
        & 0.939 & - & 0.939 & 0.946 & 0.951 (\textcolor{red}{$\uparrow$0.4\%})\\
        DihEdral &0.946&-&0.942&0.949&0.954
        & 0.950 & - & 0.944 & 0.952 & 0.955 (\textcolor{red}{$\uparrow$0.1\%})\\
        TuckER &0.953&-&0.949&0.955&0.958
        & 0.951 & - & 0.951 & 0.959 & 0.963 (\textcolor{red}{$\uparrow$0.5\%})\\
        ConvE &0.943&374 &0.935 &0.946 &0.956
        &0.945&325 &0.940 &0.949 &0.956 ($\uparrow$0.0\%)\\
        TransE &0.495&-&0.113&0.888&0.943
        &0.572&-&0.332&0.890&0.949 (\textcolor{red}{$\uparrow$0.6\%})\\
        RotatE &{0.949}&{09}&{0.944}&0.952&0.959 
         &{0.947}&{09}&{0.949}&0.960&0.966 (\textcolor{red}{$\uparrow$0.7\%  }) \\
        MQuadE &0.897&268&0.893&0.926&0.941
    &0.904&194&0.904&{0.955}&{0.964} (\textcolor{red}{$\uparrow$2.3\%  })\\
    HousE &0.954 &137 &0.948 &0.957 &0.964
    &0.955 &122 & 0.951 & 0.960 & 0.968 (\textcolor{red}{$\uparrow$0.4\%  })\\
        \bottomrule
         \end{tabular}
         \smallskip
        \caption{ Results of link prediction on the WN18 dataset.} \label{table:WN}
    \end{table}

\begin{table}[htbp]
    \tiny
    \centering
    \begin{tabular}{lcccccccccc} 
       \toprule
        & \multicolumn{5}{c}{\textbf{Original setting}} & \multicolumn{5}{c}{\textbf{CL4KGE (ours)}}\\
        \midrule
        \multirow{2}{*}{\textbf{Models}}
        &\multirow{2}{*}{\textbf{MRR}($\uparrow$)} & \multirow{2}{*}{\textbf{MR}($\downarrow$)} & \multicolumn{3}{c}{\textbf{Hits@N}($\uparrow$)}
        &\multirow{2}{*}{\textbf{MRR}($\uparrow$)} & \multirow{2}{*}{\textbf{MR}($\downarrow$)} & \multicolumn{3}{c}{\textbf{Hits@N}($\uparrow$)}
        \\
         && & \textbf{1} & \textbf{3} & \textbf{10}
         && & \textbf{1} & \textbf{3} & \textbf{10}
         \\
        \midrule
        DisMult &0.443 &4999 &0.403 &0.453 &0.534
                &0.446 &4260 &0.411 &0.466 &0.550 (\textcolor{red}{$\uparrow$1.6\%})\\
        ComplEX &0.472 &5702 & 0.432 & 0.488 &0.550
                &0.474 &4811 & 0.434 & 0.502 &0.555 (\textcolor{red}{$\uparrow$0.5\%})\\
        DihEdral &0.486&-&0.443&0.505&0.557
                &0.489&-&0.446&0.511&0.558 (\textcolor{red}{$\uparrow$0.1\%})\\
        TuckER  &0.470&-&0.443&0.482&0.526
                &0.474&-&0.448&0.502&0.557 (\textcolor{red}{$\uparrow$2.9\%})\\
        ConvE &0.430 &- &0.400&0.440  &0.520
        &0.413 &- &0.409&0.445  &0.531 (\textcolor{red}{$\uparrow$1.1\%  })\\
        TransE &0.226 &-&-&-&0.501
        &0.466&-&0.425&-&0.557 (\textcolor{red}{$\uparrow$5.6\%}) \\
        RotatE &0.476&{3340}&{0.428}&0.492&0.571
        &{0.472}&{1277}&{0.432}&{0.510}&{0.572} (\textcolor{red}{$\uparrow$0.1\%  })\\
        MQuadE &0.426&6114&0.427&0.462&0.564
        &{0.476}&{1989}&{0.412}&{0.505}&{0.584} (\textcolor{red}{$\uparrow$2.0\%})\\
        DensE &0.492 &2934 &0.443 &0.509 &0.586
        &0.495 &2002 &0.438 & 0.511 &0.588 (\textcolor{red}{$\uparrow$0.2\%})\\
        HousE &0.511 &1303 &0.465 &0.528 &0.602
        &0.515 &1299 & 0.470 & 0.532 & 0.609 (\textcolor{red}{$\uparrow$0.7\%})\\
        \bottomrule
         \end{tabular}
         \smallskip
        \caption{ Results of link prediction on the WN18RR dataset.} \label{table:WN18rr}
    \end{table}

\begin{table}[!ht]
    \centering
    \small
    \begin{tabular}{|c|c|c|c|c|c|}
         \hline & \multicolumn{3}{c|}{\text { Hit@10 }} 
         & \multirow{2}{*}{{Max}} & \multirow{2}{*}{{WO}}\\
\cline { 2 - 4 } & \text { linear } & \text { root } & \text { geometric } &&\\
\hline 
 { FB15K-237 } &0.482&\textbf{0.495}  & \textbf{0.495}   &\textbf{0.495} &0.465\\
 { WN18 } &0.944&0.946& \textbf{0.949}   & \textbf{0.949} & 0.943\\
 { WN18RR } &0.531&0.550& \textbf{0.557}   &\textbf{0.557} & 0.555\\
\hline
    \end{tabular}
    \smallskip
    \caption{Hit@10 (\%) of TransE on four benchmark datasets with different pacing functions. WO denotes the result without curriculum framework.}
    \label{tab: AblationTransE}
\end{table}

\begin{table}[!ht]
    \small
    \centering
    \begin{tabular}{|c|c|c|c|c|c|}
         \hline & \multicolumn{3}{c|}{\text { Hit@10 }}
         &
\multirow{2}{*}{{Max}}& \multirow{2}{*}{{WO}}\\
\cline { 2 - 4 } & \text { linear } & \text { root } & \text { geometric }  &&\\
\hline 
 { FB15K-237 } & 0.532 & 0.528 & \textbf{0.542}  &\textbf{0.542}&0.530\\
 { WN18 } & 0.960 & \textbf{0.966}  & \textbf{0.966}  & \textbf{0.966}&0.959\\
\text { WN18RR } & 0.570 & 0.568 & \textbf{0.572}   &\textbf{0.572}&0.571\\
\hline
    \end{tabular}
    \smallskip
    \caption{Hit@10 (\%) of RotatE on four benchmark datasets with different pacing functions. WO denotes the result without curriculum framework.}
    \label{tab: AblationRotatE}
\end{table}

\subsection{Results}\label{sec: results}

In this section, we provide the results and some discussions for \Cref{sec: experiments}.

\subsubsection{Link Prediction}


{
\Cref{table:FB15k-237}, \Cref{table:WN} and \Cref{table:WN18rr} illustrate the performance improvements of existing state-of-the-art KGE methods when integrated with the CL4KGE framework across different datasets (FB15k-237, WN18, and WN18RR). These tables collectively demonstrate the effectiveness of the curriculum learning method in enhancing various KGE models. In \Cref{table:FB15k-237}, the results on the FB15k-237 dataset show significant improvements in performance metrics (MRR, MR, Hits@1, Hits@3, and Hits@10) for models incorporating the CL4KGE framework. For instance, the Hits@10 score for the TransE model increased from 0.465 to 0.495, a 3.0\% improvement. All models exhibited improved MRR scores, such as the Dihedral model's MRR rising from 0.320 to 0.335, while MR scores generally decreased, indicating better performance. \Cref{table:WN} evaluates the same models on the WN18 dataset, where similar improvements are evident. Most models saw an increase in Hits@10, with RotatE's score rising from 0.960 to 0.966, a 0.6\% enhancement. MRR and MR improvements were consistent, with ConvE’s MRR increasing from 0.943 to 0.945. \Cref{table:WN18rr} presents the results on the more challenging WN18RR dataset, where models like TransE showed substantial gains, with Hits@10 increasing from 0.501 to 0.557, a 5.6\% rise. The MRR improvements were significant for some models, such as ComplEx's MRR increasing from 0.472 to 0.474, and Dihedral’s rising from 0.486 to 0.489. These consistent improvements across different datasets and models underscore the robustness and generalizability of the CL4KGE framework. By using Z-counts to measure triplet difficulty and employing a structured training scheduler, CL4KGE effectively enhances the training process, leading to improved performance in link prediction tasks.
}

\subsubsection{Inferring Relation Patterns}

In this subsection, we show the results on Countries datasets including countries\_S1, countries\_S2, and countries\_S3. {The queries in Countries are of the type \texttt{locatedIn}(c, ?), and the answer is one of the five regions, unlike link prediction task on knowledge graph. There are three jobs in the Countries dataset, each of which demands inferring an increasingly complex and lengthy composition pattern.} In \Cref{tab: Countries}, we summarize the results with the AUC-PR metric, which is commonly used in the literature \citep{sun2019rotate}.

\subsubsection{Ablation Study}

In this subsection, we conduct ablation studies to demonstrate the choice of pacing functions in CL4KGE. In \Cref{tab: AblationTransE} and \Cref{tab: AblationRotatE}, we evaluate the sensitivity of CL4KGE to three pacing functions: linear, root, and geometric. We compare the Hit@10 metric on the link prediction task utilizing these three pacing functions while employing TransE \citep{bordes2013translating} as the backbone of the KGE. On the most of benchmark datasets, we discover a slight advantage for the geometric pacing function.

\section{Conclusion}\label{sec: conclusion}

{In this study, we propose a curriculum framework named \textbf{CL4KGE} for knowledge graph embedding, underpinned by theoretical foundations. Our approach is inspired by the Z-shape phenomenon observed in knowledge graphs, leading us to develop a novel difficulty metric termed Z-counts. Utilizing Z-counts, we have designed a data-driven training scheduler. We conducted experiments on multiple knowledge graph benchmarks, and the results unequivocally demonstrate the efficacy of our framework.}

\bibliography{neurips}
\bibliographystyle{neurips}


\end{document}